\documentclass{article}
\pdfoutput=1

\usepackage[nonatbib,final]{mlsafety_neurips_2022}

\usepackage{amsmath,amssymb,amsfonts}
\usepackage{algorithmic}
\usepackage{graphicx}
\usepackage{textcomp}
\usepackage{xcolor}
\usepackage[mathscr]{euscript}
\usepackage{booktabs}
\usepackage{amssymb}
\usepackage{amsmath}
\usepackage[linesnumbered,ruled,vlined]{algorithm2e}

\usepackage[english]{babel}
\usepackage{amsthm}

\newtheorem*{theorem}{Theorem}

\usepackage{mlsafety_neurips_2022}

\usepackage[utf8]{inputenc} 
\usepackage[T1]{fontenc}    
\usepackage{hyperref}       
\usepackage{url}            
\usepackage{booktabs}       
\usepackage{amsfonts}       
\usepackage{nicefrac}       
\usepackage{microtype}      
\usepackage{xcolor}         

\title{Revisiting Hyperparameter Tuning \\with Differential Privacy}

\author{%
Youlong Ding$^1$\thanks{This work is done during his internship in WeBank Ltd. Co., Shenzhen, China} \;\;\;\;\;\;\;\;\;\;\; Xueyang Wu$^2$\\
$^1$ Department of Computer Science and Software Engineering, \\Shenzhen University, Shenzhen, China\\
$^2$ Department of Computer Science and Engineering, \\The Hong Kong University of Science and Technology, Hong Kong SAR, China\\
\texttt{dingyoulon@gmail.com} ~~~~ \texttt{xwuba@connect.ust.hk} \\
}

\begin{document}

\maketitle

\begin{abstract}
Hyperparameter tuning is a common practice in the application of machine learning but is a typically ignored aspect in the literature on privacy-preserving machine learning due to its negative effect on the overall privacy parameter.
In this paper, we aim to tackle this fundamental yet challenging problem by providing an effective hyperparameter tuning framework with differential privacy. 
The proposed method allows us to adopt a broader hyperparameter search space and even to perform a grid search over the whole space, since its privacy loss parameter is independent of the number of hyperparameter candidates. Interestingly, it instead correlates with the  utility gained from hyperparameter searching, revealing an explicit and mandatory trade-off between privacy and utility. Theoretically, we show that its additional privacy loss bound incurred by hyperparameter tuning is upper-bounded by the squared root of the gained utility. However, we note that the additional privacy loss bound would empirically scale like a squared root of the logarithm of the utility term, benefiting from the design of doubling step.
\end{abstract}

\section{Introduction}
Differential privacy~\cite{dwork2006calibrating,dwork2014algorithmic} has been the gold standard for quantitative and rigorous reasoning about privacy leakage from the processing of private data. Applying differential privacy to machine learning~\cite{song2013stochastic,bassily2014private,DBLP:conf/ccs/AbadiCGMMT016} is a long-lasting challenge due to the dramatic reduction in the model utility compared with the non-private version~\cite{DBLP:conf/iclr/TramerB21}. 
This motivates a bunch of work dedicated to designing private learning algorithms without sacrificing utility. However, researchers typically try different hyperparameters for best possible performance but only report the privacy parameter of a single run, which corresponds to the best accuracy achieved. As shown in~\cite{papernot2022hyperparameter}, the choice of hyperparameter would cause leakage of private information such as membership inference~\cite{shokri2017membership}. This finding is aligned with the theory of differential privacy if applied strictly. Supposed one run of private learning to be $\varepsilon$-DP, if we repeat the learning process $K$ times using different hyperparameters and select the model with the highest accuracy, the privacy loss parameter would be increased to $O(k\varepsilon)$ with basic composition or $\tilde{O}(\sqrt{k}\varepsilon)$ with advanced composition~\cite{dwork2010boosting,kairouz2015composition}, essentially a multiple of the original privacy loss $\varepsilon$.
Several existing works attempt to handle such embarrassment. The stability-based approach~\cite{chaudhuri2013stability} leverages the stability assumption of the learning algorithm for improving the privacy loss bounds. RandTune~\cite{10.1145/3313276.3316377,papernot2022hyperparameter} proposes to introduce another level of uncertainty for sharpening the privacy bounds. Concretely, it first draws $K$ from a geometric distribution, then randomly and independently picks $K$ hyperparameters, and runs a private training algorithm for each selected hyperparameter. Then the total privacy parameter is shown to be bounded by 2 or 3 times of the privacy parameter for a single run. However, the largest challenge in RandTune is to guarantee the success probability of picking the best hyperparameter within only $K$ independent trials, which gets more difficult when the search space of hyperparameters is large. Therefore, a recent work~\cite{DBLP:conf/aaai/MohapatraSH0022} leverages adaptive optimizers to reduce the potential hyperparameter space. Nevertheless, the number of trials remains to be unpredictable, and therefore it is still hard to guarantee the model quality.



In this paper, we propose a framework for hyperparameter tuning with differential privacy, which 1) satisfies rigorous guarantees with differential privacy and 2) allows us to perform a grid search over all candidate hyperparameters.
To achieve that, we enable the overall privacy loss to be independent of both hyperparameter search space and the original privacy parameter of DP-SGD. It instead depends on the final utility of the model (i.e., accuracy on the validation set), which is an interesting property because it reveals an explicit and mandatory trade-off between privacy and utility. Additionally, our approach does not require each training run for hyperparameter selection to be differentially private. 
This property can be utilized for efficient tuning of hyperparameters whose behavior is consistent between private and non-private versions.
Comparisons between different methods are in Table~\ref{tab:comp}.





\begin{table}[!ht]
\centering
\caption{Comprehensive comparison between different methods for hyperparameter tuning with DP.}
\label{tab:comp}
\begin{tabular}{ccccc}
\toprule
Method & Full search &  \# runs  & Priv. run per search & Privacy budget\\ \midrule

Naive & Yes &  $|\mathcal{S}|$ & Required & $\tilde{O}(\sqrt{|\mathcal{S}|}\varepsilon$) \\
RandTune\cite{papernot2022hyperparameter,10.1145/3313276.3316377}& No &  $K\sim \mathcal{D}^{\bigstar}$ & Required & $2\varepsilon$ or $3\varepsilon$\\
Ours & Yes &  $|\mathcal{S}|$ & Not required & $\varepsilon+\tilde{O}( \sqrt{\frac{u^*-u_0}{g}}) \varepsilon_0$ \\\bottomrule
\end{tabular}
 \begin{flushleft}$\bigstar$: $K$ is drawn from a Truncated Negative Binomial Distribution or a Poisson distribution. \end{flushleft} 
\end{table}

\begin{algorithm}[!htp]
    \SetAlgoLined 
	\caption{\textbf{Propose-test Hyperparameter Tuning with Doubling Step}}
	\label{alg:doubling}
	\KwIn{Hyperparameter candidates set $\mathcal{S}$; Privacy parameters $\varepsilon, \varepsilon_0>0, \delta \in(0, 1)$; Training set $\mathcal{D}_{train}$; Validation set $\mathcal{D}_{valid}$; Number of partitions  $k$; Utility granularity $g\in (0,1)$; Utility lower bound $u_0\in [0,1)$.}
	\KwOut{Model parameters $\theta^*$ along with the selected hyperparameters $s^*\in \mathcal{S}$.}  
	\BlankLine
	Initialize $s^* \leftarrow \emptyset$, $u \leftarrow u_0$, $step \leftarrow 1$;
	
	Initialize $count \leftarrow 0$ \tcp*{Help variable for analysis}
	
	Partition the training set $\mathcal{D}_{train}$ into disjoint subsets $\mathcal{D}_i$, $k=1,2,...,k$;
	
	\ForEach{s = $1, 2, ..., |\mathcal{S}|$}{
	   Initialize $u_s \leftarrow 0$;
		    
		\ForEach{$i$ = $1, 2,..., k$}{
            $\theta \leftarrow {{\rm Train}}(\mathcal{D}_i, \mathcal{S}_s)$ or ${{\rm PrivateTrain}}(\mathcal{D}_i, \mathcal{S}_s, \varepsilon, \delta)$;
        	    
            $u_s^{(i)} \leftarrow {{\rm Acc}}(\mathcal{D}_{valid}, \theta)$ \tcp*{Compute utility for each partition}

        }
        $u_s \leftarrow \frac{1}{k}\sum_{i=1}^{k} u_s^{(i)}$ \tcp*{Compute average utility across partitions}
		     
	}
	
	
	\While{$step \neq 0$}{
		$count \leftarrow count + 1$
		
	    Let $\hat{u} \leftarrow u + step\times g + {{\rm Lap(2}}/(k\varepsilon_0))$ \tcp*{Propose a utility threshold}
        
        $isAccumulated \leftarrow \textbf{False}$;
        
	    \ForEach{s = $1, 2, ..., |\mathcal{S}|$}{
            
            $\gamma_s \leftarrow {{\rm Lap(4}}/(k\varepsilon_0))$ \tcp*{Sample noise}
            
		    \If{$u_s + \gamma_s \geq \hat{u}$}
		    { 
		        $s^* \leftarrow s$ \tcp*{Record the current hyperparameter}
		        
		        $u \leftarrow u + step\times g$ \tcp*{Accumulate utility}
		        
		        $isAccumulated \leftarrow \textbf{True}$;
		        
		        $step \leftarrow step \times 2$  \tcp*{Double-expand current step}
		        
		        \textbf{break;}
		    }
        }
        
        \If{not $isAccumulated$}{
            $step \leftarrow \lfloor step/2 \rfloor$ \tcp*{Double-compress current step}
		 }
		         
        \If{u $\geq$ 1}{
            \textbf{break} \tcp*{Required for termination guarantee}
		 }

	}
	
	$\theta^* \leftarrow {{\rm PrivateTrain}}(\mathcal{D}_{train}, S_{s^*}, \varepsilon, \delta)$ \tcp*{Using the selected hyperparameter}
\end{algorithm}


\section{Propose-test Hyperparameter Tuning with Doubling Step}
Our method described in Algorithm~\ref{alg:doubling} is inspired by a bunch of classical algorithms. Specifically, we inherit the AboveThreshold component (line 18) of Sparse Vector Techinique~\cite{dwork2009complexity} to check whether the current candidate is eligible for utility accumulation. 
We also apply Subsample and Aggregate~\cite{nissim2007smooth} (lines 3-11) to help us obtain a proxy utility function with relatively low sensitivity (also used by PATE~\cite{DBLP:conf/iclr/PapernotAEGT17,DBLP:conf/iclr/PapernotSMRTE18}). To potentially decrease the number of iterations (and thus the privacy loss parameter), we use the doubling algorithms (lines 22 and
27) originally designed to solve Lowest-Common-Ancestor (LCA) in the tree. Note that the naive combination of doubling and randomized algorithm may cause the termination problem.
To accommodate this, we add an additional threshold checking (line 16) to ensure termination. The overall privacy guarantee is as follows.

\begin{theorem}
The algorithm~\ref{alg:doubling} guarantees ($\varepsilon+\tilde{O}(\sqrt{\frac{u^*-u_0}{g}}) \varepsilon_0, \delta$)-differential privacy for all $\varepsilon_0$, $\varepsilon_1 > 0$, $\delta \in (0, 1)$, $\mu_0 \in [0, 1)$, and $g\in (0,1)$.
\end{theorem}

\begin{proof}
The execution of the main loop (line 12) can be treated as running a sequence of procedures $\mathcal{M}_1$, $\mathcal{M}_2$, ...,  $\mathcal{M}_t$, ..., $\mathcal{M}_T$, where $\mathcal{M}_t: \mathcal{D}_{train}, \mathcal{S}, \{u_s^{(t)}\}, u^{(t)} \rightarrow s^{(t)}, u^{(t+1)}$. Fix any two neighbouring training set $\mathcal{D}$ and $\mathcal{D'}$, and let the outputs on them (with the same set of hyperparameters $\mathcal{S}$, $\{u_s^{(t)}\}$, and $u^{(t)}$) be $\mathcal{A}$ and $\mathcal{A}'$, respectively. We first prove that every single mechanism $\mathcal{M}_t$ is differential private, and the result follows from the advanced composition.

Firstly, we bound the $\ell_1$ sensitivity of $u_s^{(t)}(\mathcal{D})$, which is the maximum change in $\ell_1$ norm caused by adding or removing one sample from $\mathcal{D}$. Observe that the output accuracy score for each split is bounded by 1, and therefore, the maximum change caused by one split is at most $1/k$. Secondly, suppose the output $\mathcal{A}$ is $s^{(t)}=k$, we define $u_\mathcal{S}(\mathcal{D})= {{\rm max}}_{s<k}(u_s^{(t)}(\mathcal{D}) + \gamma_s^{(t)})$, representing the maximum noisy utility of all hyperparameters tried on $\mathcal{D}$.
We then fix the values of $\{\gamma_s^{(t)} | s<k\}$. That is, we assume the two runs on $\mathcal{D}$ and $\mathcal{D}'$ share the same value of noise assigned for the corresponding hyperparameter candidate's utility. Note that although this will weaken the privacy protection effect (but easy for analysis) since we somewhat reduce the amount of the uncertainty underlying the algorithm, as we will show later, it is still sufficient to obtain the desired privacy loss bound. 
After fixing, the randomness on the output is over $\tilde{u}^{(t)}$ and $\gamma_k^{(t)}$. 
The probability that $\mathcal{M}^{(t)}$ on $\mathcal{D}$ outputs $\mathcal{A}^{(t)}$ can be bounded as follows. Let $\Delta u_\mathcal{S} = u_\mathcal{S}(\mathcal{D}) - u_\mathcal{S}(\mathcal{D}’)$, $\Delta u_k^{(t)}=u_k^{(t)}(\mathcal{D}') - u_k^{(t)}(\mathcal{D})$, then
\begin{equation}
\setlength{\arraycolsep}{1pt}
\begin{split}
    p(&\mathcal{A}^{(t)} = \{s^{(t)}, u^{(t)}\}) 
    = p( u_\mathcal{S}(\mathcal{D}) < \tilde{u}^{(t)}\leq u_k^{(t)}(\mathcal{D}) + \gamma_k^{(t)})\\
    =&\int_{\tilde{u}}
    \int_{\gamma_k}
    p(\tilde{u}^{(t)}=\tilde{u})
    \cdot  p(\gamma_k^{(t)}=\gamma_k)
    \cdot \textbf{1}[u_\mathcal{S}(\mathcal{D}) < \tilde{u} \leq u_k^{(t)}(\mathcal{D}) + \gamma_k]d\gamma_k d\tilde{u}\\
    =&
    \int_{\tilde{u}}
    \int_{\gamma_k}
    p(\tilde{u}^{(t)}=\tilde{u}+\Delta u_\mathcal{S})
    \cdot  p(\gamma_k^{(t)}=\gamma_k + \Delta u_\mathcal{S}+\Delta u_k^{(t)}) \\ 
    &~~\cdot \textbf{1}[u_\mathcal{S}(\mathcal{D}) < \tilde{u}+\Delta u_\mathcal{S} \leq u_k^{(t)}(\mathcal{D}) + (\gamma_k+ \Delta u_\mathcal{S}+\Delta u_k^{(t)})]d\gamma_k d\tilde{u}\\
    \leq& \int_{\tilde{u}}
    \int_{\gamma_k}
    e^{\varepsilon_0 / 2}p(\tilde{u}^{(t)}=\tilde{u})
    \cdot  e^{\varepsilon_0 / 2}p(\gamma_k^{(t)}=\gamma_k)
    \cdot \textbf{1}[u_\mathcal{S}(\mathcal{D}') < \tilde{u} \leq u_k^{(t)}(\mathcal{D}') + \gamma_k]d\gamma_k d\tilde{u}\\
    \leq& e^{\varepsilon_0} p(\mathcal{A}^{'(t)} = \{s^{(t)}, u^{(t)}\}) 
\end{split}
\end{equation}
Therefore, each mechanism $\mathcal{M}_t$ is $\varepsilon_0$-differential private. In the following, we bound $T$, which is the total number of mechanisms in the sequence, i.e., the number of iterations. 
Although it is  difficult to analyze the behavior of the doubling step in the randomized setting, we fortunately find that the worst-case scenario is easy to attain. If we record the value of the variable $isAccumulated$ over each iteration, the outcome sequence in the worst case will be
\begin{equation}
    \{{{\rm True, False, True, False, True, False, ...}}\}
\end{equation}
This behavior corresponds to the zig-zag phenomenon in the $step$-$u$ plot (See the top part in Figure~\ref{fig:side:a}). 
In the worst case, the utility only gains $g$ every two iterations from the beginning, and we know that the utility is upper-bounded by $1$. Thus, the total number of iterations can be bounded by $O(\frac{u^*-u_0}{g})$. 
Using advanced composition, the additional privacy loss is bounded by $\tilde{O}(\sqrt{\frac{u^*-u_0}{g}})\varepsilon_0$.
Finally, we run the DP-SGD using the best parameter with the privacy parameter of ($\varepsilon$, $\delta$). From basic composition,  the total privacy parameter of the algorithm is thus ($\varepsilon+\tilde{O}(\sqrt{\frac{u^*-u_0}{g}}) \varepsilon_0, \delta$).
\end{proof}

\textbf{Discussion.} We have adopted a worst-case analysis to bound $T$, i.e., the number of iterations. Therefore, the privacy loss bound is upper-bounded by the final result. However, benefiting from the incorporation of the doubling step, the empirical privacy loss bound is even better. The intuition is that the doubling step would bound $T$ by $O({{\rm log_2}}\frac{u^*-u_0}{g})$, if the algorithm was deterministic.
We then run a simulation to show the actual privacy loss bound. The number of partitions $k$ is set to 10. The additional privacy parameter $\varepsilon_0$ is set to 0.1. We set the number of hyperparameters to 100 and randomly draw 100 samples from the uniform distribution over $(0, 1)$, corresponding to each $u_s$. We set $g$ to 0.01. Figure~\ref{fig:side:a} visualizes the doubling step process. Each point ($u$, $step$) represents the state after an iteration. ($u$, $step$) will start at (0, 1) at the beginning, and after each iteration, it will either move left by a half or move right-top by ($step$, $u$ + $step$)). The algorithm terminates once it reaches the line of $step=0$ or $u=1$. 
The red short line at the left bottom shows the behavior of the upper bound within the same iterations.
We can see that the utility accumulates exponentially at the beginning and converges fast w.h.p (over 99.9\% here).
We then vary the value of $\frac{u^*-u_0}{g}$ to plot the corresponding empirical bound for $T$ in Figure~\ref{fig:side:b}. Its dependency is almost logarithmic with $\frac{u^*-u_0}{g}$.

\begin{figure}
  \begin{minipage}[t]{0.45\linewidth}
    \centering
    \includegraphics[scale=0.4]{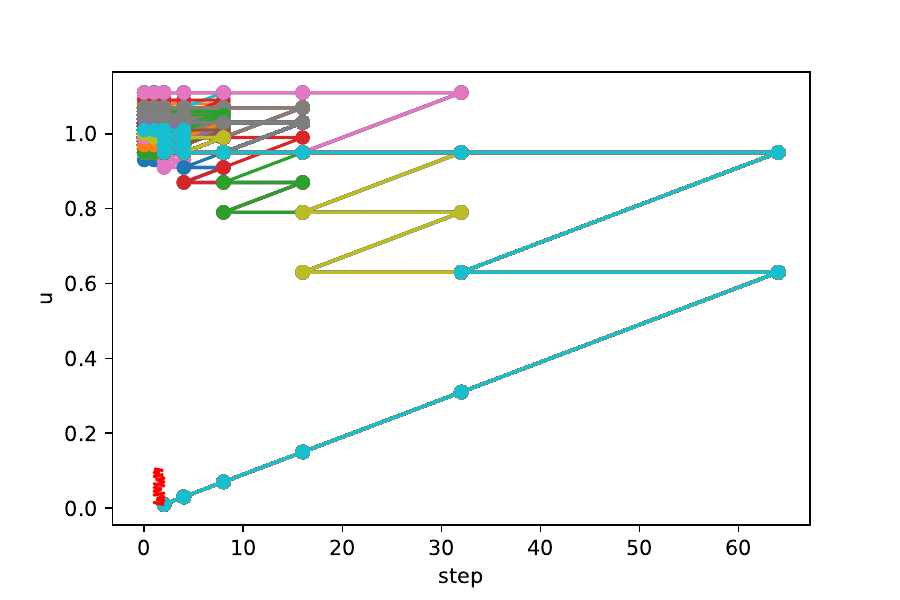}
    \caption{Visulaization of $step$-$u$ line over each iteration (starting from left bottom, ending at left top) for 1000 different random seeds (shown in color, except for red).}
    \label{fig:side:a}
  \end{minipage}%
  \hspace{5mm}
  \begin{minipage}[t]{0.45\linewidth}
    \centering
    \includegraphics[scale=0.4]{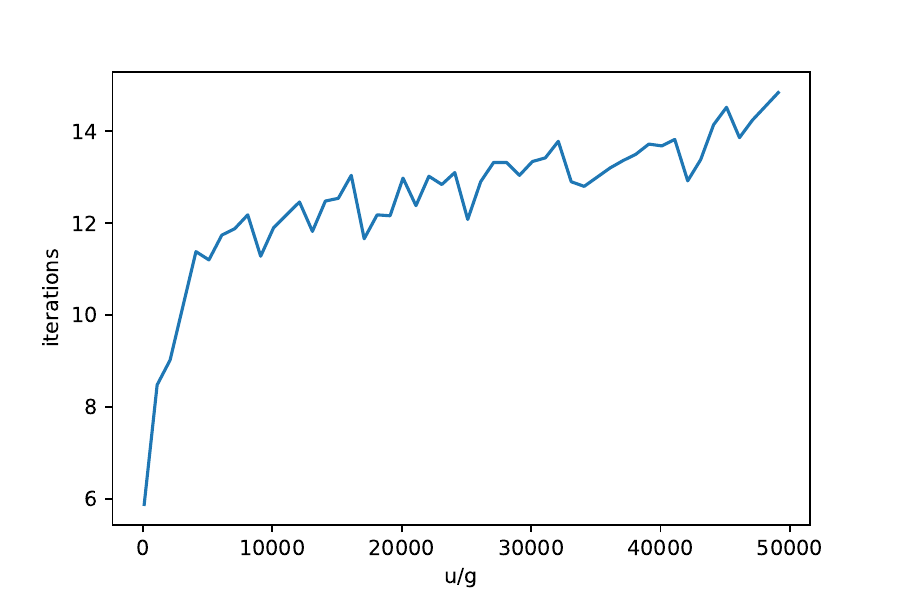}
    \caption{Plot of \#iterations versus $\frac{u^*-u_0}{g}$. As can be seen, in practice, it is nearly a logarithmic dependency.}
    \label{fig:side:b}
  \end{minipage}
\end{figure}

\section{Conclusion}
In this paper, we propose an ML algorithm-agnostic framework for hyperparameter tuning with differential privacy under rigorous privacy guarantees. Compared to existing differentially private hyperparameter tuning methods that suffer from large hyperparameter search space, our additional privacy loss parameter is free from the size of the hyperparameter candidates set and the original privacy parameter of DP-SGD. 
Instead, it correlates with the final utility of the tuned model and is upper-bounded by the squared root of the utility term.
Therefore, it allows us to perform hyperparameter tuning on a larger range, even with a grid search, leading to potentially higher utility. 
We believe that our work would be meaningful in the field of privacy-preserving machine learning, and would be valuable for future research in this area.

\newpage
\appendix



\end{document}